%% file: paper.tex
\documentclass{article} % For LaTeX2e
\usepackage{iclr2016_conference,times}
\usepackage{url}

\usepackage[hang]{footmisc}
\usepackage{graphicx}
\usepackage{amsmath,amssymb,amsthm}

\iclrfinalcopy

\newtheorem{definition}{Definition}

\newtheorem{theorem}[definition]{Theorem}
\newtheorem{remark}[definition]{Remark}

\title{ACDC: A Structured Efficient Linear Layer}

\author{
Marcin Moczulski$^1$ \quad Misha Denil$^{1}$ \quad Jeremy Appleyard$^2$ \quad Nando de Freitas$^{1,3}$\\[0.5cm]
$^1$University of Oxford  \quad $^2$NVIDIA \quad $^3$CIFAR\\
\texttt{marcin.moczulski@stcatz.ox.ac.uk}\\
\texttt{misha.denil@gmail.com}\\
\texttt{jappleyard@nvidia.com}\\
\texttt{nando.de.freitas@cs.ox.ac.uk}
}

\newcommand{\pderiv}[2]{\frac{\partial {#1}}{\partial {#2}}}

\newcommand{\C}{\mathbb{C}}

\input{mydef}  % Macros from Kevin Murphy's book.

\begin{document}

\maketitle

\begin{abstract}
The linear layer is one of the most pervasive modules in deep learning representations. However, it requires $O(N^2)$ parameters and $O(N^2)$ operations. These costs can be prohibitive in mobile applications or prevent scaling in many domains. Here, we introduce a deep, differentiable, fully-connected neural network module composed of 
diagonal matrices of parameters, $\vA$ and $\vD$, and the discrete cosine transform $\vC$. The core module, structured as $\vA\vC\vD\vC^{-1}$, has $O(N)$ parameters and incurs $O(N\log N)$ operations. We present theoretical results showing how deep cascades of $\operatorname{ACDC}$ layers approximate linear layers. $\operatorname{ACDC}$  is, however, a stand-alone module and can be used in combination with any other types of module. In our experiments, we show that it can indeed be successfully interleaved with ReLU modules in convolutional neural networks for image recognition. 
Our experiments also study critical factors in the training of these structured modules, including initialization and depth. 
Finally, this paper also points out avenues for implementing the complex version of $\operatorname{ACDC}$ using photonic devices.   
\end{abstract}

\section{Introduction}
\label{sec:introduction}

{\let\thefootnote\relax\footnotetext{Torch implementation of ACDC is available at \url{https://github.com/mdenil/acdc-torch}}}

The linear layer is the central building block of nearly all modern neural network models.  A notable exception to this is the convolutional layer, which has been extremely successful in computer vision; however, even convolutional networks typically feed into one or more linear layers after processing by convolutions.  Other specialized network modules including LSTMs \citep{hochreiter1997long}, GRUs \citep{cho2014learning}, the attentional mechanisms used for image captioning~\citep{xu2015} and machine translation \citep{bahdanau2015neural}, reading in Memory Networks \citep{sukhbaatar2015}, and both reading and writing in Neural Turing Machines \citep{graves2015}, are all built from compositions of linear layers and nonlinear modules, such as sigmoid, softmax and ReLU layers.

The linear layer is essentially a matrix-vector operation, where the input $\vx$ is scaled with a matrix of parameters $\vW$ as follows:
\be
\vy = \vx \vW 
\ee
When the number of inputs and outputs is $N$, the number of parameters stored in $\vW$ is $O(N^2)$. It also takes $O(N^2)$ operations to compute the output $\vy$.

In spite of the ubiquity and convenience of linear layers, their $O(N^2)$ size is extremely wasteful. Indeed, several studies focusing on feedforward perceptrons and convolutional networks have shown that the parametrisation of linear layers is extremely wasteful, with up to 95\% of the parameters being redundant~\citep{DenilSDRF13,gong2014compressing,SainathKSAR13}.

Given the importance of this research topic, we have witnessed a recent explosion of works introducing structured efficient linear layers (SELLs).  We adopt the following notation to describe SELLs within a common framework:
\be
\vy = \vx \vPhi = \vx \vPhi(\vD,\vP,\vS,\vB)
\ee
We reserve the capital bold symbol $\vD$ for diagonal matrices, $\vP$ for permutations, $\vS$ for sparse matrices, and $\vB \in \{\vF,\vH,\vC\}$ for bases such as Fourier, Hadamard and Cosine transforms respectively. In this setup, the parameters are typically in the diagonal or sparse entries of the matrices $\vD$ and $\vS$. Sparse matrices aside, the computational cost of most SELLs is $O(N\log N)$, while the number of parameters is reduced from $O(N^2)$ to a mere $O(N)$. These costs are a consequence of the facts that we only need to store the diagonal matrices, and that the Fourier, Hadamard or Discrete Cosine transforms can be efficiently computed in $O(N\log N)$ steps.  

Often the diagonal and sparse matrices have fixed random entries. When this is the case, we will use tildes to indicate this fact (\emph{e.g.}, $\tilde{\vD}$).

Our first SELL example is the Fast Random Projections method of \citet{Ailon2009}: 
\be
\vPhi =  \tilde{\vD} \vH \tilde{\vS}
\ee
Here, the sparse matrix $\tilde{\vS}$ has Gaussian entries, the diagonal $\tilde{\vD}$ has $\{+1,-1\}$ entries drawn independently with probability $1/2$, and $\vH$ is the Hadamard matrix.
The embeddings generated by this SELL preserve metric information with high probability, as formalized by the theory of random projections. 

Fastfood \citep{le:2013}, our second SELL example, extends fast random projections as follows:
\be
\vPhi = \tilde{\vD}_1 \vH \vP \tilde{\vD}_2 \vH \tilde{\vD}_3.
\ee
In \citep{Yang2014}, the authors introduce an adaptive variant of Fastfood, with the random diagonal matrices replaced by diagonal matrices of parameters, and show that it outperforms the random counterpart when applied to the problem of replacing one of the fully connected layers of a convolutional neural network for ImageNet \citep{jia2014caffe}. Interestingly, while the random variant is competitive in simple applications (MNIST), the adaptive variant has a considerable advantage in more demanding applications (ImageNet). 

The adaptive SELLs, including Adaptive Fastfood and the alternatives discussed subsequently, are end to end differentiable. They require only $O(N)$ parameters and $O(N\log N)$ operations in both the forward and backward passes of backpropagation. These benefits can be achieved both at train and test time.

\citet{cheng2015exploration} introduced a SELL consisting of the product of a circulant matrix ($\vR$) and a random diagonal matrix ($\tilde{\vD}_1$). Since circulant matrices can be diagonalized with the discrete Fourier transform \citep{Golub:1996}, this SELL falls within our general notation: 
\be
\vPhi =  \tilde{\vD}_1 \vR  = \tilde{\vD}_1 \vF \vD_2 \vF^{-1}. 
\ee
\citet{sindhwani2015structured} introduced a Toeplitz-like structured transform, within the framework of displacement operators. Since Toeplitz matrices can be ``embedded'' in circulant matrices, they can also be diagonalized with the discrete Fourier transform \citep{Golub:1996}.

In this work, we introduce a SELL that could be thought of as an adaptive variant of the method of \citet{cheng2015exploration}. In addition, instead of using a (single) shallow SELL as in previous works \citep{Yang2014,cheng2015exploration,sindhwani2015structured}, we consider deep SELLs:
\be
\vPhi =  \prod_{k=1}^{K} \vA_k \vF \vD_k \vF^{-1}. 
\label{eq:deepsell}
\ee
Here, $\vA$ is also a diagonal matrix of parameters, but we use a different symbol to emphasize that $\vA$ scales the signal in the original domain while $\vD$ scales it in the Fourier domain.

While adaptive SELLs perform better than their random counterparts in practice, there is a lack of theory for adaptive SELLs. Moreover, the empirical studies of recent adaptive SELLs have many deficiencies. For instance, it is often not clear how performance varies depending on implementation, and many critical details such as initialization and the treatment of biases are typically obviated. In addition, the gains are often demonstrated in models of different size, making objective comparison very difficult.

In addition to demonstrating good performance replacing the fully connected layers of CaffeNet, we present a theoretical approximation guarantee for our deep SELL in Section~\ref{sec:deepsell}.
We also discuss the crucial issue of implementing deep SELLs efficiently in modern GPU architectures in Section~\ref{sec:gpu} and release this software with this paper. This engineering contribution is important as many of the recently proposed methods for accelerating linear layers often fail to take into account the attributes and limitations of GPUs, and hence fail to be adopted.

\subsection{Lightning fast deep SELL}

\begin{figure}
\centering
\includegraphics[width=0.75\linewidth]{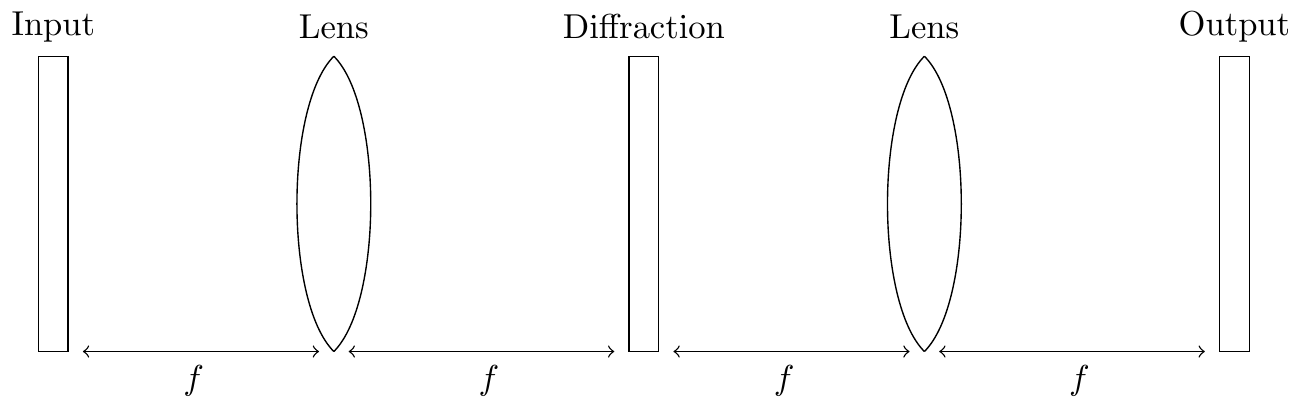}
\caption{Example of a $4f$ system.  This system implements the multiplication of an optical signal by a circulant matrix $\vF\vD\vF^{-1}$.  The lenses apply Fourier transforms to the signal and the diffraction element applies the diagonal multiplication.}
\label{fig:fourf}
\end{figure}

Our deep SELL (equation~(\ref{eq:deepsell})) offers several possibilities for analog physical implementation. Given the great demand for fast low energy neural networks, the possibility of harnessing physical phenomena to perform efficient computation in deep networks is worthy of consideration.    

In the Fourier optics field, it is well known that the two-dimensional Fourier transform can be implemented with a paraxial optical system consiting of a lens of focal length $f$ in free space. In this setup, known as a $2f$ system, a waveform in the frontal focal plane of the lens, viewed as a two-dimensional complex array, is transformed to another one in the focal plane behind the lens that corresponds to the Fourier transform of the array. A $4f$ system is obtained by placing a diffractive element in between two $2f$ systems at a distance $f$ from each (shown in Figure~\ref{fig:fourf}).

Every circulant matrix $\vR = \vF \vD \vF^{-1}$ can be realized optically using a $4f$ system, with the transformation by the diffractive optical device corresponding to the multiplication by the complex diagonal matrix $\vD$ 
 \citep{reif97,muller1998algorithmic,Huhtanen08,Schmid:2000}. Moreover, paraxial diffractive optical systems with consecutive products of circulant and diagonal matrices can factor a complex matrix into products of diagonal and circulant matrices \citep{muller1998algorithmic,Huhtanen15}. Hence, in principle the mapping of equation~(\ref{eq:deepsell}) can be implemented with optical elements. 
 
In a separate research community, \citet{Hermans:2015} recently discussed using waves in a trainable medium for learning linear layers by backpropagation, and suggested a potential implementation using an integrated photonics chip. The nanophotonic chip consists of a cascade of unitary trasformations of the optical signals interleaved with tuneable waveguides (phase shifters). \citet{Hermans:2015} present an abstraction of this chip. In particular, if we let $\vo$ and $\vo'$ represent the optical fields at the input and output waveguides, the chip implements the following transformation:
\be
\vo' =  \prod_{k=1}^{K} \vD_k \vU_k \vo 
\label{eq:nanophotonic}
\ee
where $\vU_k$ is a unitary transformation of the signal and $\vD_k$ is a diagonal matrix $\vD_k = \diag(\exp(j\boldsymbol{\varphi}_k))$ with tuneable phase shifts $\boldsymbol{\varphi}_k$. 
By restricting the diagonal matrices in equation~(\ref{eq:deepsell}) to be of this complex form, the circulant $\vR = \vF \vD \vF^{-1}$ is unitary and we obtain an equivalence between equations~(\ref{eq:deepsell}) and (\ref{eq:nanophotonic}). This points to a potential nanophotonic implementation of our complex deep SELL. 

More recently, \citet{Saade:2015} disclosed an invention that peforms optical analog random projections.

\section{Further related works}
\label{sec:related}

The literature on this topic is vast, and consequently this section only aims to capture some of the significant trends. We refer readers to the related work sections of the papers cited in the previous and present section for further details.

As mentioned earlier, many studies have shown that the parametrisation of linear layers is extremely wasteful \citep{DenilSDRF13,gong2014compressing,SainathKSAR13}.  In spite of this redundancy, there has been little success in improving the linear layer, since natural extensions, such as low rank factorizations, lead to poor performance when trained end to end. For instance, \citet{SainathKSAR13} demonstrate significant improvements in reducing the number of parameters of the output softmax layers, but only modest improvements for the hidden linear layers.  

Several methods based on low-rank decomposition and sparseness have been proposed to eliminate parameter redundancy at test time, but they provide only a partial solution as the full network must be instantiated during training~\citep{Collins2014,XueLG13,blundell-uncertainty-2015,Liu_2015_CVPR,Han2015}. That is, these approaches require training the original full model. Hashing techniques have been proposed to reduce the number of parameters \citep{Chen2015,Bakhtiary15}. Hashes have irregular memory access patterns and, consequently, good performance on large GPU-based platforms is an open problem. Distillation \citep{Hinton15,Romero2015} also offers a way of compressing neural networks, as a post-processing step.

\cite{novikov2015tensorizing} use a multi-linear transform (Tensor-Train decomposition) to attain significant reductions in the number of parameters in some of the linear layers of convolutional networks.

\section{Deep SELL}
\label{sec:deepsell}

We define a single component of deep SELL as
$\operatorname{AFDF}(\vx) = \vx \vA\vF\vD\vF^{-1}$, where $\vF$ is the Fourier transform  and $\vA, \vD$ are complex diagonal matrices. It is straightforward to see that the AFDF transform is not sufficient to express an arbitrary linear operator $\vW \in \C^{n\times n}$.  An AFDF transform has $2n$ degrees of freedom, whereas an arbitrary linear operator has $n^2$ degrees of freedom.

To this end, we turn our attention to studying compositions of AFDF transforms.  By composing AFDF transforms we can boost the number of degrees of freedom, and we might expect that any linear operator could be constructed as a composition of sufficiently many AFDF transforms.  In the following we show that this is indeed possible, and that a bounded number of AFDF transforms is sufficient.

\begin{definition}
  The order-$K$ AFDF transformation is the composition of $K$ consecutive AFDF operations with (optionally) different $\vA$ and $\vD$ matrices.  We write an order-$K$ complex AFDF transformation as follows
  \begin{align}
  \label{eq:afdfk}
    \vy = \operatorname{AFDF}_K(\vx) = \vx \left[\prod_{k=1}^K \vA_k\vF\vD_k\vF^{-1} \right]. 
  \end{align}
  We also assume, without loss of generality, that $\vA_1 = \vI$ so that $\operatorname{AFDF}_1(\vx) = \vx \vF\vD_1\vF^{-1}$. 
  %is a convolution of $\vx$ by $\operatorname{diag}(\vD)$ (TODO: not exactly, actually by $diag(D)F^{-1})$).
\end{definition}

For the analysis it will be convenient to rewrite the AFDF transformation in a different way, which we refer to as the \emph{optical presentation}.

\begin{definition}
  If $\vy = \operatorname{AFDF}_K(\vx)$ then we define the optical presentation of an order-$K$ AFDF transform as
  \begin{align*}
    \hat{\vy} = \hat{\vx} \left[ \prod_{k=1}^{K-1} \vD_k \vR_{k+1} \right ] \vD_K
  \end{align*}
  where $\hat{\vx}$ and $\hat{\vy}$ are the Fourier transforms of $\vx$ and $\vy$, and $\vR_{k+1} = \vF^{-1}\vA_{k+1}\vF$.
\end{definition}

\begin{remark}
  The matrix $\vR = \vF^{-1}\vA\vF$ is circulant.  This follows from the duality between convolution in the spatial domain and pointwise multiplication in the Fourier domain.
\end{remark}

The optical presentation shows how the spectrum of $\vx$ is related to the spectrum of $\vy$.  Importantly, it shows that we can express an order-$K$ AFDF transform as a linear operator in Fourier space that is composed of a product of circulant and diagonal matrices.  Transformations of this type are well studied in the Fourier optics literature, as they can be realized with cascades of lenses.

Of particular relevance to us is the main result of \citet{Huhtanen15} which states that almost all (in the Lebesgue sense) matrices $\vM \in \C^{N\times N}$ can be factored as
\begin{align*}
  \vM &= \left[\prod_{i=1}^{N-1} \vD_{2i-1} \vR_{2i}\right] \vD_{2N-1}
\end{align*}
where $\vD_{2j-1}$ is diagonal and $\vR_{2j}$ is circulant.  This factorization corresponds exactly to the optical presentation of an order-$N$ AFDF transform, therefore we conclude the following:
\begin{theorem}
  An order-$N$ AFDF transform is sufficient to approximate any linear operator in $\C^{N\times N}$ to arbitrary precision.
\end{theorem}
\begin{proof}
  Every AFDF transform has an optical presentation, and by the main result of \citet{Huhtanen15} operators of this type are dense in $\C^{N\times N}$.
\end{proof}

% TODO: Add extension to the theory to allow for permutations / invertible matrices.

\section{ACDC: A practical deep SELL}
\label{sec:acdc}

Thus far we have focused on a complex SELL, where theoretical guarantees can be obtained.  In practice we find it useful to consider instead a real SELL.  The real version of $\operatorname{AFDF}_K$, denoted $\operatorname{ACDC}_K$ has the same form as Equation~(\ref{eq:afdfk}), with complex diagonals replaced with real diagonals, and Fourier transforms replaced with Cosine Transforms.  This change departs from the theory of Section~\ref{sec:deepsell}; however, our experiments show that this does not appear to be a problem in practice.

The reasons for considering ACDC over AFDF are purely practical.  
\begin{enumerate}
\item Most existing deep learning frameworks support only real numbers, and thus working with real valued transformations simplifies the interface between our SELL and the rest of the network.
\item Working with complex numbers effectively doubles the memory footprint of of the transform itself, and more importantly, of the activations that interact with it.
\end{enumerate}
The importance of the second point should not be underestimated, since the computational complexity of our SELL is quite low, a typical GPU implementation will be bottlenecked by the overhead of moving data through the GPU memory hierarchy.  Reducing the amount of data to be moved allows for a significantly faster implementation.  We discuss these concerns in more detail in Section~\ref{sec:gpu}.

% TODO: Mention what to do when not square.

% Thus far we have focused on the forward pass of linear layers $\vy_i = \vx_i \vW $. For bacpropagation, we also need the backward messages:
% \bea
% \pderiv{L}{\vW}
% %= \pderiv{L}{\vy_i} \pderiv{\vy_i}{\vW}
% = (\pderiv{L}{\vy_i} \vx_i)^{T} \\
% %= \vx_i^T \pderiv{L}{\vy_i}  \\
% \pderiv{L}{\vx_i}
% %= (\pderiv{L}{\vy_i} \pderiv{\vy_i}{\vx_i})^{T}
% = (\vW \pderiv{L}{\vy_i})^T
% \eea
% We again encounter the $O(N^2)$ costs in the backward pass.

% An ACDC layer provides an approximation to the square linear layer:
% \be
% %used to be in a diff order, change of notation
% %\vy_i = \vA\vC^{-1}\vD\vC \vx_i
% \vy_i = \vx_i \vA\vC\vD\vC^{-1} 
% \ee
% where $\vA$ and $\vD$ are diagonal matrices and $\vC$ is the discrete cosine transform (DCT) and $\vC^{-1}$ is the inverse DCT.

% Since $\vA$ and $\vD$ are diagonal storage of these matrices costs $O(N)$. The computational cost is also reduced to $O(N\log N)$.

In this work, we use the DCT (type II) matrix with entries
\be
c_{nk} = \sqrt{\frac{2}{N}}
\left[ \epsilon_k \cos \left(
\frac{\pi(2n+1)k}{2N}
\right)
\right]
\ee
for $n,k = 0,1,\ldots,N$, and where $\epsilon_k = 1/\sqrt{2}$ for $k=0$ or $k=N$ and $\epsilon_k = 1$ otherwise. DCT matrices are real and orthogonal: $\vC^{-1}=\vC^T$. Moreover, the DCTs are separable transforms. That is, the DCT of a multi-dimensional signal can be decomposed in terms of successive DCTs of the appropriate one-dimensional components of the signal. The DCT can be computed efficiently using the Fast Fourier Transform (FFT) algorithm (or the specialized fast cosine transform).

% While one can express the DCT in terms of the (complex) discrete Fourier transform (DFT), the DCT uses (real) cosine bases, and concentrates the energy of the signal in a small number of coefficients \cite{Lam:2000}. For this reason, it is often preferred to the DFT for lossy compression applications, such as JPEG image compression, MP3 audio compression and MPEG video compression.

% Since $\vA$ and $\vD$ are diagonal and $\vC$ is the DCT, storage of these matrices costs $O(N)$. The computational cost is also reduced to $O(N\log N)$ because we can use the FFT algorithm to compute $\vC \vx_i$.

Denoting $\vh_1 = \vx_i \vA$, $\vh_2 = \vh_1 \vC$, $\vh_3 = \vh_2 \vD$, $\vy_i = \vh_3 \vC^{-1}$, 
and $\vA = \operatorname{diag}(\va), \vD = \operatorname{diag}(\vd)$
we have the following derivatives in the backward pass:

\bea
\pderiv{L}{\vd}
= \pderiv{\vy_i}{\vd} \pderiv{L}{\vy_i}
= \pderiv{ \vh_2 \vD  }{\vd} \pderiv{ \vh_3 \vC^{-1}}{\vh_3} \pderiv{L}{\vy_i}
= \operatorname{diag}(\vh_2) \vC \pderiv{L}{\vy_i}
= \vh_2 \odot \vC \pderiv{L}{\vy_i} \\
\pderiv{L}{\vD} = \operatorname{diag}(\pderiv{L}{\vd}) \\
\pderiv{L}{\va}
= \pderiv{\vy_i}{\va} \pderiv{L}{\vy_i}
%= \pderiv{ \vx_i \vA  }{\va} \pderiv{ \vh_1 \vC  }{\vh_1} \pderiv{ \vh_2 \vD }{\vh_2} \pderiv{ \vh_3 \vC^{-1} }{\vh_3} \pderiv{L}{\vy_i}
= \pderiv{ \vx_i \vA  }{\va} \pderiv{ \vh_1 \vC  }{\vh_1} \pderiv{ \vh_2 \vD }{\vh_2} \pderiv{ L }{\vh_3} 
= \vx_i \odot \vC^{-1} \vd \odot \vC \pderiv{L}{\vy_i} \\
\pderiv{L}{\vA} = \operatorname{diag}(\pderiv{L}{\va}) \\
\pderiv{L}{\vx_i}
= \pderiv{\vy_i}{\vx_i} \pderiv{L}{\vy_i}
%= \pderiv{ \vx_i \vA }{\vx_i} \pderiv{ \vh_1 \vC  }{\vh_1} \pderiv{ \vh_2 \vD  }{\vh_2} \pderiv{ \vh_3 \vC^{-1} }{\vh_3} \pderiv{L}{\vy_i}
= \pderiv{ \vx_i \vA }{\vx_i} \pderiv{ L  }{\vh_1} 
= \va \odot \vC^{-1} \vd \odot \vC \pderiv{L}{\vy_i}
\eea

\section{Efficient implementation of ACDC}
\label{sec:gpu}

The processor used to benchmark the ACDC layer was an NVIDIA Titan X. The peak floating point throughput of the Titan X is 6605 GFLOPs, and the peak memory bandwidth is 336.5GB/s\footnote{\url{http://www.geforce.co.uk/hardware/desktop-gpus/geforce-gtx-titan-x/specifications}}. This gives an arithmetic intensity (FLOPs per byte) of approximately 20. In the ideal case, where there is enough parallelism for the GPU to hide all latencies, an algorithm with a higher arithmetic intensity than this would be expected to be floating point throughput bound, while an algorithm with lower arithmetic intensity would be expected to be memory throughput bound.

The forward pass of a single example through a size-$N$ ACDC layer when calculated using 32-bit floating point arithmetic requires at least $24N$ bytes to be moved to and from main memory. Eight bytes per element for each of $\vA$ and $\vD$, four bytes per element for the input, and four bytes per element for the output. It also requires approximately $4N + 5 N\log_2(N)$ floating point operations\footnote{\url{http://www.fftw.org/speed/method.html}}. When batching, the memory transfers for $\vA$ and $\vD$ are expected to be cached as they are reused for each example in the batch, so for the purposes of calculating arithmetic intensity in the batched case it is reasonable to discount them. The arithmetic intensity of a minibatch passing through an ACDC layer is therefore approximately:

\begin{align*}
AI = (4 + 5 \log_2(N)) / 8
\end{align*}

For the values of $N$ we are interested in ($128 \longrightarrow 16,384$) this arithmetic intensity varies between 4.9 and 9.3, indicating that the peak performance of a large ACDC layer with a large batch size is expected to be limited by the peak memory throughput of the GPU (336.5GB/s), and that optimization of an ACDC implementation should concentrate on removing any extraneous memory operations. 
 
Two versions of ACDC have been implemented. One performs the ACDC in a single call, with the minimum of $8N$ bytes moved per layer (assuming perfect caching of $\vA$ and $\vD$). The other performs ACDC with multiple calls, with significantly more than $8N$ bytes moved per layer.

\begin{figure}
\includegraphics[width=0.47\linewidth]{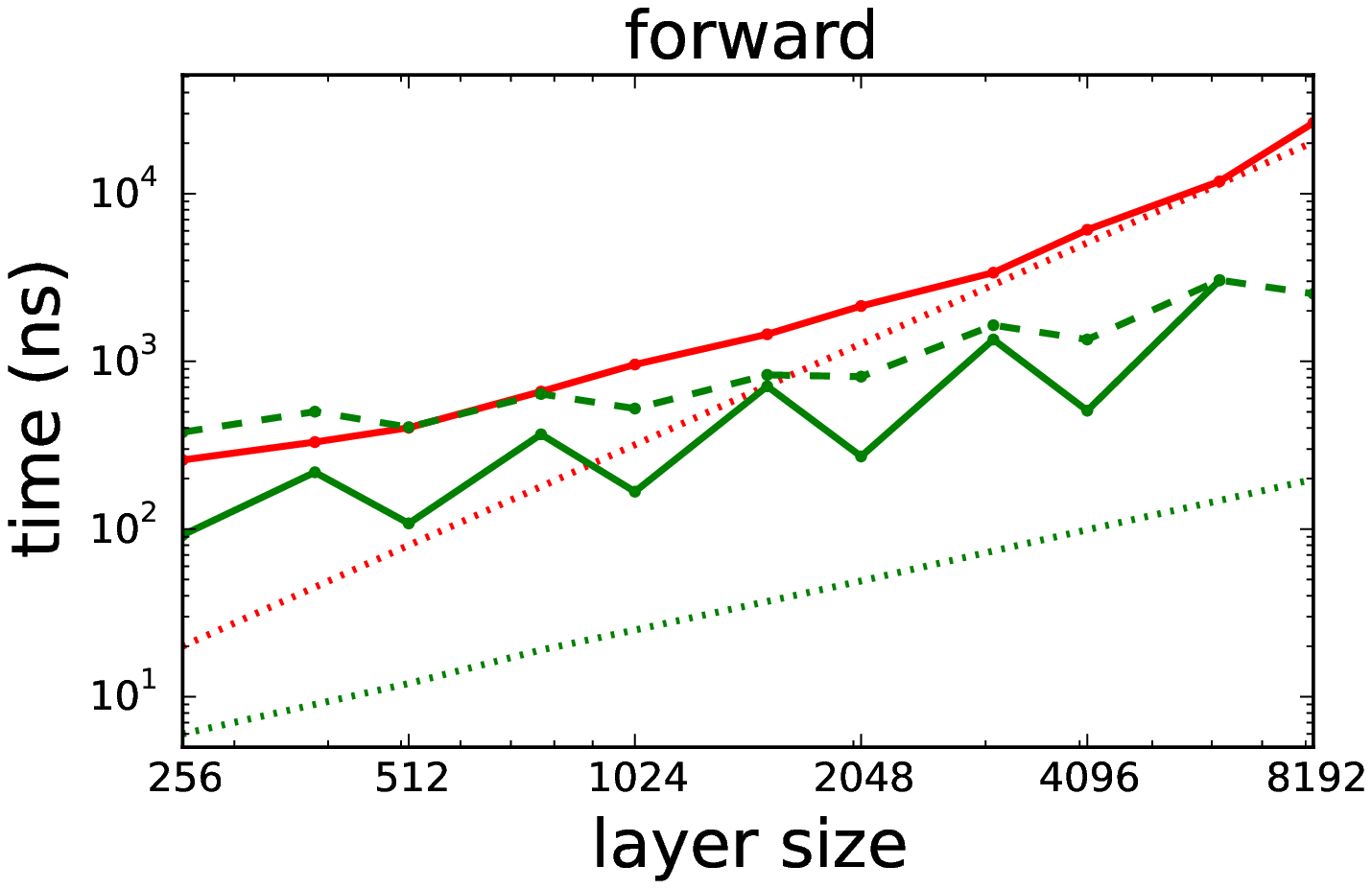}
\includegraphics[width=0.47\linewidth]{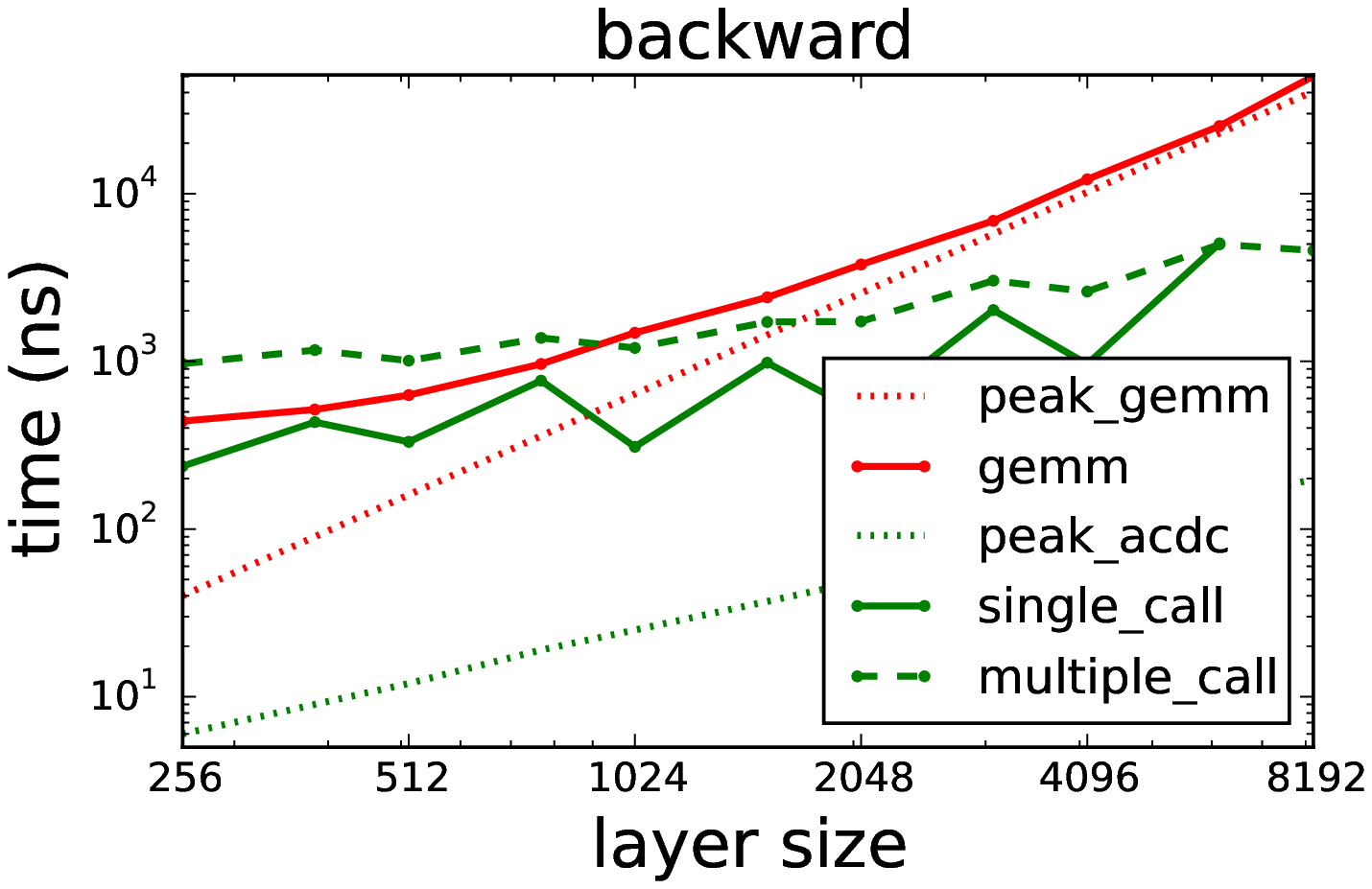}
\caption{Performance comparison of theoretical and actual performance of our ACDC implementations to an ordinary dense linear layer using a batch size of 128.  Peak curves show maximum theoretical performance achievable by the hardware.}
\label{fig:acdc_speed}
\end{figure}
 
\subsection{Single call implementation}
 
To minimize traffic to and from main memory intermediate loads or stores during the layer must be eliminated. To accomplish this kernel fusion is used to fuse all of the operations of ACDC into a single call, with intermediate values being stored in temporary low-level memory instead of main memory. This presents two challenges to the implementation.

Firstly, the size of the ACDC layer is limited by the availability of temporary memory on the GPU. This limits the size of the ACDC layer that can be calculated. It also has performance implications: the temporary memory used to store intermediate values in the computation is shared with the registers required for basic calculation, such as loop indices. The more of this space that is used by data, the fewer threads can fit on the GPU at once, limiting parallelism.
 
Secondly, the DCT and IDCT layers must be written by hand so that they can be efficiently fused with the linear layers. Implementations of DCT and IDCT are non-trivial, and a generic implementation able to handle any input size would be a large project in itself. For this reason, the implementation is constrained to power-of-two and multiples of large power-of-two layer sizes.

\subsection{Multiple call implementation}

While expected to be less efficient a multiple call implementation is both much simpler programmatically, and much more generically usable. Using the method of \citet{makhoul1980fast} it is possible to perform size-$N$ DCTs and IDCTs using size-$N$ FFTs. As such, the NVIDIA library cuFFT can be used to greatly simplify the code required, as well as achieve reasonable performance across a wide range of ACDC sizes. The procedure is as follows:
\begin{enumerate}
\item Multiply input by $\vA$ and set up $\vC_1$
\item Perform $\vC_1$ using a C2C cuFFT call 
\item Finalize $\vC_1$, multiply by $\vD$ and setup $\vC_2$
\item Perform $\vC_2$ using a C2C cuFFT call
\item Finalize $\vC_2$
\end{enumerate}
 
The total memory moved for this implementation is significantly higher as each call requires a load and a store for each element. The performance trade-off with the single call method is therefore one of parallelism against memory traffic.

\subsection{Performance Comparison}

Figure~\ref{fig:acdc_speed} compares the speed of the single and multiple call implementations of ACDC against dense matrix-matrix multiplication for a variety of layer sizes.

It is clear that in both the forward and backward pass ACDC layers have a significantly lower runtime than fully connected layers using dense matrices. Even if the matrix-matrix operations were running at peak, ACDC still would outperform them by up to 10 times.

As expected, the single call version of ACDC outperforms the multiple call version, although for smaller layer sizes the gap is larger. When the layer size increases the multiple call version suffers significantly more from small per-call overheads. Both single and multiple call versions of ACDC perform significantly worse on non power-of-two layer sizes. This is because they rely on FFT operations, which are known to be more efficient when the input sizes are of lengths $z^n$, where $z$ is a small integer\footnote{\url{http://docs.nvidia.com/cuda/cufft/\#accuracy-and-performance}}. 

While the backward pass of ACDC is expected to take approximately the same time as the forward pass, it takes noticeably longer. To compute the parameter gradients one needs the input into the $\vD$ operation and the gradient of the output from the $\vA$ operation. As the aim of the layer is to reduce memory footprint it was decided instead to recompute these during the backward pass, increasing runtime while saving memory.

\section{Experiments}
\label{sec:experiments}

\subsection{Linear layers}

In this section we show that we are able to approximate linear operators using ACDC as predicted by the theory of Section~\ref{sec:deepsell}.  These experiments serve two purposes
\begin{enumerate}
\item They show that recovery of a dense linear operator by SGD is feasible in practice.  The theory of Section~\ref{sec:deepsell} guarantees only that it is possible to approximate any operator, but does not provide guidance on how to find this approximation.  Additionally, \cite{Huhtanen15} suggest that this is a difficult problem.
\item They validate empirically that our decision to focus on ACDC over the complex AFDF does not introduce obvious difficulties into the approximation.  The theory provides guarantees only for the complex case, and the experiments in this section suggest that restricting ourselves to real matrices is not a problem.
\end{enumerate}

We investigate using ACDC on a synthetic linear regression problem
\be 
 \vY = \vX \vW_{true} + \pmb{\epsilon},
\ee
where $\vX$ of size $10,000 \times 32$ and $\vW_{true}$ of size $32 \times 32$ are both constructed by sampling their entries uniformly at random in the unit interval.
Gaussian noise $ \pmb{\epsilon} \sim \mathcal{N} (0,10^{-4})$  is added to the generated targets.

The results of approximating the operator $\vW_{true}$ using $\operatorname{ACDC}_K$ for different values of $K$ are shown in Figure~\ref{fig:acdc_linear_train_loss}. 
The theory of Section 3 predicts that, in the complex case, for a $32\times 32$ matrix it should be sufficient to have 32 layers of ACDC to express an arbitrary $\vW_{true}$.

We found that initialization of the matrices $\vA$ and $\vD$ to identity $\vI$, with Gaussian noise $\mathcal{N}(0,10^{-2})$ added the diagonals in order to break symmetries, is essential for models having many ACDC layers. (We found the initialization to be robust to the specification of the noise added to the diagonals.) 

The need for thoughtful initialization is very clear in  Figure~\ref{fig:acdc_linear_train_loss}. With the right initialization (leftmost plot), the approximation results of Section 3 are confirmed, with improved accuracy as we increase the number of ACDC layers. However, if we use standard strategies for initializing linear layers (rightmost plot), we observe very poor optimization results as the number of ACDC layers increases.

This experiment suggests that fewer layers suffice to arrive at a reasonable approximation of the original $\vW_{true}$ than what the theory guarantees.  With neural networks in mind this is a very relevant observation.  It is well known that the linear layers of neural networks are compressible, indicating that we do not need to express an arbitrary linear operator in order to achieve good performance. Instead, we need only express a sufficiently interesting subset of matrices, and the result with 16 ACDC layers points to this being the case.

In Section~\ref{sec:convnet} we show that by interspersing nonlinearities between ACDC layers in a convolutional network it is possible to use dramatically fewer ACDC layers than the theory suggests are needed while still achieving good performance.

\begin{figure}[bt]
\centering
\includegraphics[width=0.495\linewidth]{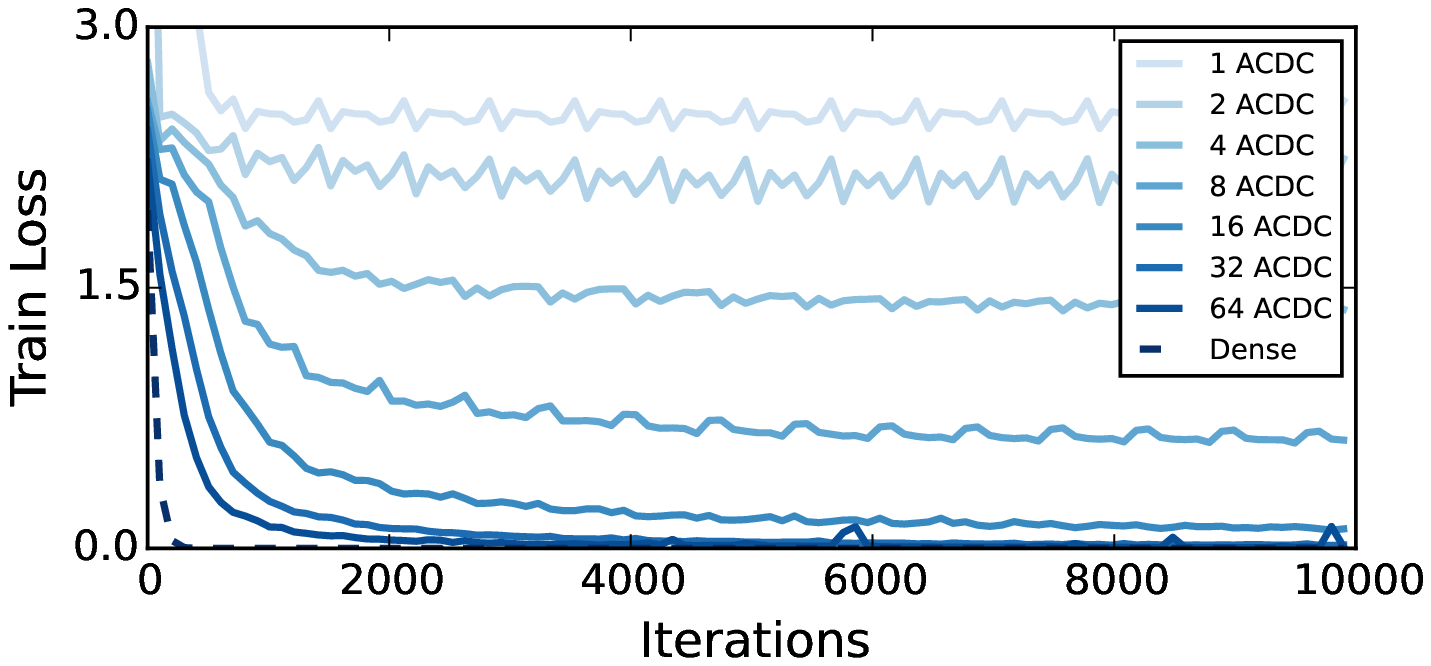}
\includegraphics[width=0.495\linewidth]{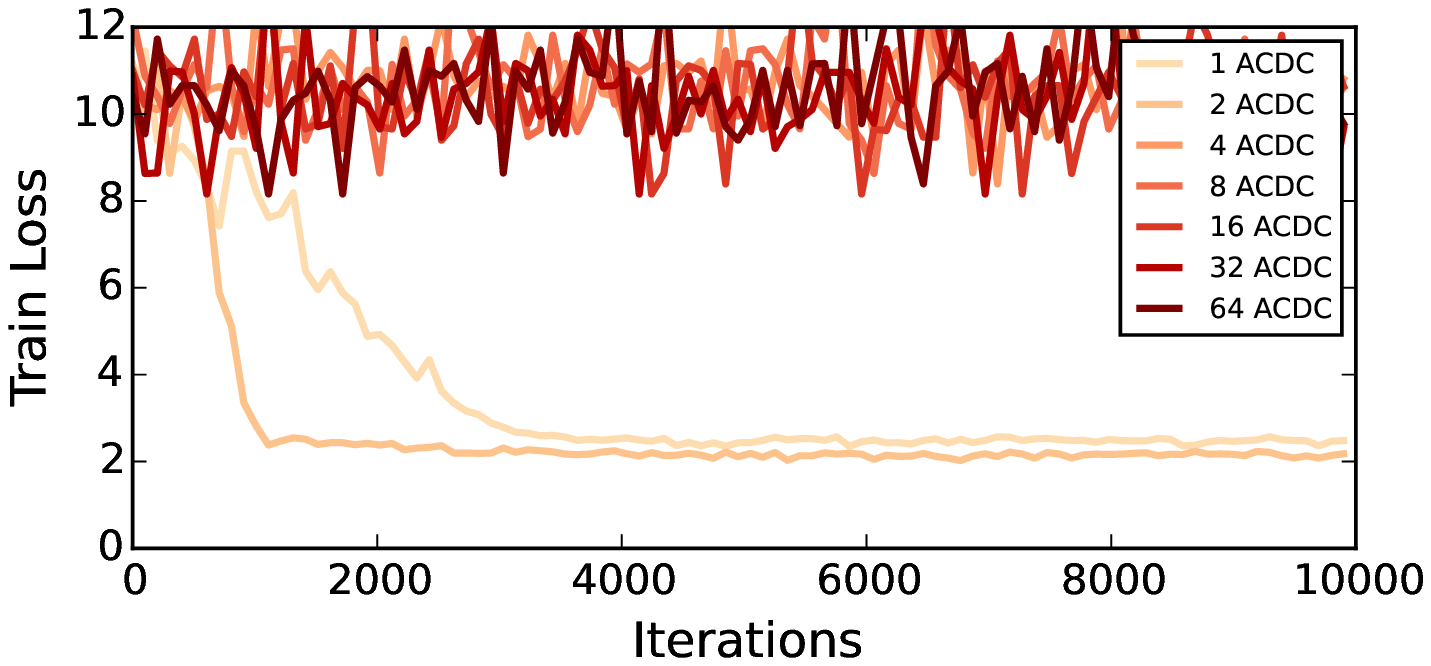}
\caption{Training loss for different number of ACDC layers compared to loss for the dense matrix. \textbf{Left:} Initialization: $\mathcal{N} (1,\sigma^{2})$ with $\sigma = 10^{-1}$. \textbf{Right:} Initialization: $\mathcal{N} (0,\sigma^{2})$ with $\sigma = 10^{-3}$.  Note the difference in scale on the y-axis.}
\label{fig:acdc_linear_train_loss}
\end{figure}

\subsection{Convolutional networks}
\label{sec:convnet}

In this section we investigate replacing the fully connected layers of a deep convolutional network with a cascade of ACDC layers.  In particular we use the CaffeNet architecture\footnote{\url{https://github.com/BVLC/caffe/tree/master/models/bvlc_reference_caffenet}} for ImageNet~\citep{imagenet_cvpr09}.  We target the two fully connected layers located between features extracted from the last convolutional layer
and the final logistic regression layer, which we replace with $12$ stacked ACDC transforms interleaved with ReLU non-linearities and permutations.  The permutations assure that adjacent SELLs are incoherent.

The model was trained using the SGD algorithm 
with learning rate $0.1$ multiplied by $0.1$ every $100{,}000$ iterations,
momentum $0.65$ and weight decay $0.0005$.  The output from the last convolutional layer was scaled by $0.1$, and the learning rates for each matrix $\vA$ and $\vD$ were multiplied by $24$ and $12$.
All diagonal matrices were initialized from $\mathcal{N} (1, 0.061)$ distribution.
No weight decay was applied to $\vA$ or $\vD$.
 Additive biases were added to the matrices $\vD$, but not to $\vA$, as this sufficed to provide the ACDC layer with a bias terms
just before the ReLU non-linearities. Biases were initialized to $0$.  To prevent the model from overfitting dropout regularization was placed 
before each of the last 5 SELL layers with dropout probability equal to $0.1$.

The resulting model arrives at $43.26\%$ error which is only $0.67\%$ worse 
when compared to the reference model, so
SELL confidently stays within $1\%$  of the performance of the original network.  We report this result, as well as a comparison to several other works in Table~\ref{tab:imagenet_comparison}.

\begin{table}[t]
  \begin{center}
  \begin{tabular}{l|r|r|r|r|r}
    Test Time Post-Processing & Top-1 Err Increase & \# of Param & Reduction \\
    \hline
  % Authors compare to AlexNet model, we used Caffe. AlexNet model has a slightly worse performance (42.59% vs 42.8%)    
  \textit{\cite{Collins2014}} & \textit{1.81\%} & \textit{15.2M} & \textit{x4.0} \\
  \cite{Han2015} & 0.00\% & 6.7M & x9 \\
  \cite{DBLP:journals/corr/HanMD15} (P+Q) & 0.00\% & $\sim$2.3M & ~ x27\\
  \hline~\\
  Train and Test Time Reduction \\
  \hline
  %Circulant CNN 2
  \cite{cheng2015exploration} (Circulant CNN 2) & 0.40\% & $>$ 16.3M & $<$ x3.8 \\
  %Tensorizing Nets (TT4 FC FC)
  $^*$\cite{novikov2015tensorizing} (TT4 FC FC) & 0.30\% &  - & x3.9 \\ % E67 (updated)
  %Tensorizing Nets (TT4 TT4 FC)
  $^*$\textit{\cite{novikov2015tensorizing} (TT4 TT4 FC)} & \textit{1.30\%} &  - & \textit{x7.4} \\ % E67 (updated)
  %Finetuned SVD 1 
  \cite{Yang2014} (Finetuned SVD 1) & 0.14\% & 46.6M & x1.3 \\
  %Finetuned SVD 2
  \textit{\cite{Yang2014} (Finetuned SVD 2)} & \textit{1.22\%} & \textit{23.4M} & \textit{x2.0} \\
  %Adaptive Fastfood 16
  \cite{Yang2014} (Adaptive Fastfood 16) & 0.30\% &  16.4M & x3.6 \\ % E68-70 (updated)
  %Adaptive Fastfood 32
  % \cite{Yang2014} (Adaptive Fastfood 32) & $-$0.66\% & 32.8M & x1.8 \\
  \textbf{ACDC} & \textbf{0.67\%} & \textbf{9.7M} & \textbf{x6.0} \\
  \hline
  % CaffeNet Reference Model & 42.59\% & 58.7M & - \\ 
  CaffeNet Reference Model & 0.00\% & 58.7M & x1.0 \\ 
\end{tabular}
\end{center}
\caption{ Comparison of SELL with alternative factorization methods achieving marginal performance drop on the \mbox{ImageNet} dataset. Entries in italics incur an increase in top-1 error of $>$1.0\%.  Entries marked with a star use VGG16, which makes them not directly comparable to our own.  Previous works have shown that it is typically possible to achieve $\sim 30\%$ greater compression factors on VGG16 than on AlexNet-style architectures \citep{DBLP:journals/corr/HanMD15,Han2015}.}
\label{tab:imagenet_comparison}
\end{table}

\begin{figure}
\centering
\includegraphics[width=0.65\linewidth]{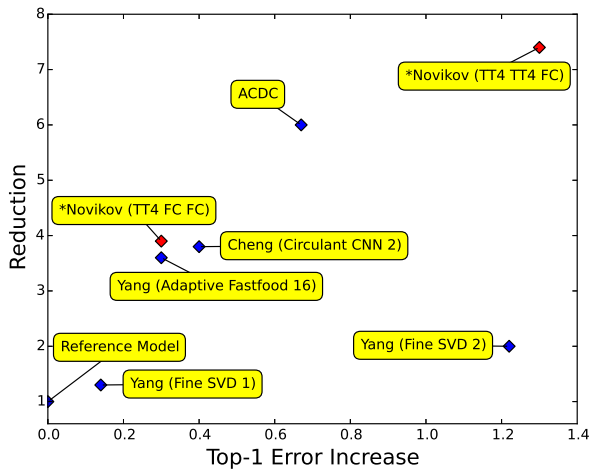}
\caption{Visual comparison of the tradeoff between parameter and accuracy reduction for train time applicable SELLs. Red entries (marked with a star in the labels) use VGG16, which makes them not directly comparable to the others, as discussed in the caption of Table~\ref{tab:imagenet_comparison}.}
\end{figure}

The two fully connected layers of CaffeNet, consisting of more than $41$ million parameters, are replaced
with SELL modules which contain a combined $165,888$ parameters. These results agree with the hypothesis that
neural networks are over-parameterized formulated by 
\cite{DenilSDRF13} and supported by \cite{Yang2014}. At the same time such a tremendous
reduction without significant loss of accuracy suggests that SELL is a powerful 
concept and a way to use parameters efficiently.

This approach is an improvement over Deep Fried Convnets~\citep{Yang2014} 
and other FastFood~\citep{le:2013} based
transforms in the sense that the layers remain narrow and become deep
(potentially interleaved with non-linearites) as opposed to
wide and shallow, while maintaining comparable or better performance.  The result of narrower layers is that the final softmax classification layer requires substantially fewer parameters, meaning that the resulting compression ratio is higher.

Our experiment shows that ACDC transforms are an attractive building block 
for feedforward convolutional architectures, that can be used as a structured alternative to fully connected layers, while fitting very well into the deep learning philosophy of introducing transformations executed
in steps as the signal is propagated down the network rather than projecting to higher-dimensional
spaces.

% As SELL for neural networks is a novel proposition there is no known configuation 
% and architecture standards established yet
% and with our limited hyperparameters search we conjecture that even better results could be achieved 
% on ImageNet problem
% with more adaptive optimization algorithm like ADAM or 
% RMSprop~\citep{Tieleman2012}, which we leave for future work.

% We compare performance of SELL with other factorization methods for compressing ImageNet networks in Table~\ref{tab:imagenet_comparison}.
%

It should be noted that the method of pruning proposed in \citep{Han2015} and the follow-up method 
of pruning, quantizing and Huffman coding proposed in \citep{DBLP:journals/corr/HanMD15} achieve 
compression rates between x9 and x27 on AlexNet\footnote{\cite{DBLP:journals/corr/HanMD15} report x35 compression by using Huffman coding and counting bytes. We report the number of parameters here for consistency.} by applying a pipeline of reducing operations on a trained
models. Usually it is necessary to perform at least a few iterations of such reductions to arrive at the stated compression rates. For the AlexNet model one such iteration takes 173 hours according to \citep{Han2015}.
On top of that as this method requires training the original full model the time cost of that operation should be taken into consideration as well.

Compressing pipelines target models that are ready for deployment and function in the environment where 
amount of time spent on training is absolutely dominated by the time spent evaluating predictions.  In contrast, SELL methods are appropriate for incorporation into the design of a model.

\section{Conclusion}

We introduced a new Structured Efficient Linear Layer, which adds to a growing literature on using memory efficient structured transforms as efficient replacements for the dense matrices in the fully connected layers of neural networks. The structure of our SELL is motivated by matrix approximation results from Fourier optics, but has been specialized for efficient implementation on NVIDIA GPUs.

We have shown that proper initialization of our SELL allows us to build very deep cascades of SELLs that can be optimized using SGD.  Proper initialization is simple, but is essential for training cascades of SELLs with more than a few layers.  Working with deep and narrow cascades of SELLs makes our networks more parameter efficient than previous works using shallow and wide cascades because the cost of layers interfacing between the SELL and the rest of the network is reduced (e.g.\ the size of the input to the dense logistic regression layer of the network is much smaller).

In future work we plan to investigate replacing the diagonal layers of ACDC with other efficient structured matrices such as band or block diagonals.  These alternatives introduce additional parameters in each layer, but may give us the opportunity to explore the continuum between depth and expressive power per layer more precisely.

Another interesting avenue of investigation is to include SELL layers in other neural network models such as RNNs or LSTMs.  Recurrent nets are a particularly attractive targets as they are typically composed entirely of linear layers.  This means that the potential parameter savings are quite substantial, and since the computational bottleneck is in these models comes from matrix-matrix multiplications there is a potential speed advantage as well.



{\small
\bibliographystyle{iclr2016_conference}
\bibliography{deepbib}
}

\end{document}

%% file: mydef.tex
\newcommand{\myvec}[1]{\mathbf{#1}}
\newcommand{\myvecsym}[1]{\boldsymbol{#1}}

\newcommand{\vPhi}{\myvecsym{\Phi}}

\newcommand{\va}{\myvec{a}}

\newcommand{\vd}{\myvec{d}}

\newcommand{\vh}{\myvec{h}}

\newcommand{\vo}{\myvec{o}}

\newcommand{\vx}{\myvec{x}}

\newcommand{\vy}{\myvec{y}}

\newcommand{\vA}{\myvec{A}}
\newcommand{\vB}{\myvec{B}}
\newcommand{\vC}{\myvec{C}}
\newcommand{\vD}{\myvec{D}}

\newcommand{\vF}{\myvec{F}}

\newcommand{\vH}{\myvec{H}}
\newcommand{\vI}{\myvec{I}}

\newcommand{\vM}{\myvec{M}}

\newcommand{\vP}{\myvec{P}}

\newcommand{\vR}{\myvec{R}}
\newcommand{\vS}{\myvec{S}}

\newcommand{\vU}{\myvec{U}}

\newcommand{\vW}{\myvec{W}}
\newcommand{\vX}{\myvec{X}}

\newcommand{\vY}{\myvec{Y}}

\newcommand{\diag}{\mathrm{diag}}

\newcommand{\be}{\begin{equation}}
\newcommand{\ee}{\end{equation}}
\newcommand{\bea}{\begin{eqnarray}}
\newcommand{\eea}{\end{eqnarray}}
\newcommand{\beaa}{\begin{eqnarray*}}
\newcommand{\eeaa}{\end{eqnarray*}}

\DeclareMathAlphabet{\mathpzc}{OT1}{pzc}{m}{n}

%% file: paper.bbl
\begin{thebibliography}{36}
\providecommand{\natexlab}[1]{#1}
\providecommand{\url}[1]{\texttt{#1}}
\expandafter\ifx\csname urlstyle\endcsname\relax
  \providecommand{\doi}[1]{doi: #1}\else
  \providecommand{\doi}{doi: \begingroup \urlstyle{rm}\Url}\fi

\bibitem[Ailon \& Chazelle(2009)Ailon and Chazelle]{Ailon2009}
Ailon, Nir and Chazelle, Bernard.
\newblock The {Fast Johnson Lindenstrauss Transform} and approximate nearest
  neighbors.
\newblock \emph{SIAM Journal on Computing}, 39\penalty0 (1):\penalty0 302--322,
  2009.

\bibitem[Bahdanau et~al.(2015)Bahdanau, Cho, and Bengio]{bahdanau2015neural}
Bahdanau, Dzmitry, Cho, Kyunghyun, and Bengio, Yoshua.
\newblock Neural machine translation by jointly learning to align and
  translate.
\newblock In \emph{International Conference on Learning Representations}, 2015.

\bibitem[Bakhtiary et~al.(2015)Bakhtiary, Lapedriza, and Masip]{Bakhtiary15}
Bakhtiary, Amir~H., Lapedriza, {\`{A}}gata, and Masip, David.
\newblock Speeding up neural networks for large scale classification using
  {WTA} hashing.
\newblock \emph{arXiv preprint arXiv:1504.07488}, 2015.

\bibitem[Blundell et~al.(2015)Blundell, Cornebise, Kavukcuoglu, and
  Wierstra]{blundell-uncertainty-2015}
Blundell, Charles, Cornebise, Julien, Kavukcuoglu, Koray, and Wierstra, Daan.
\newblock Weight uncertainty in neural networks.
\newblock In \emph{ICML}, 2015.

\bibitem[Chen et~al.(2015)Chen, Wilson, Tyree, Weinberger, and Chen]{Chen2015}
Chen, Wenlin, Wilson, James~T., Tyree, Stephen, Weinberger, Kilian~Q., and
  Chen, Yixin.
\newblock Compressing neural networks with the hashing trick.
\newblock In \emph{ICML}, 2015.

\bibitem[Cheng et~al.(2015)Cheng, Yu, Feris, Kumar, Choudhary, and
  Chang]{cheng2015exploration}
Cheng, Yu, Yu, Felix~X, Feris, R, Kumar, Sanjiv, Choudhary, Alok, and Chang,
  Shih-Fu.
\newblock An exploration of parameter redundancy in deep networks with
  circulant projections.
\newblock In \emph{ICCV}, 2015.

\bibitem[Cho et~al.(2014)Cho, Van~Merri{\"e}nboer, Gulcehre, Bahdanau,
  Bougares, Schwenk, and Bengio]{cho2014learning}
Cho, Kyunghyun, Van~Merri{\"e}nboer, Bart, Gulcehre, Caglar, Bahdanau, Dzmitry,
  Bougares, Fethi, Schwenk, Holger, and Bengio, Yoshua.
\newblock Learning phrase representations using rnn encoder-decoder for
  statistical machine translation.
\newblock In \emph{Empiricial Methods in Natural Language Processing}, 2014.

\bibitem[Collins \& Kohli(2014)Collins and Kohli]{Collins2014}
Collins, Maxwell~D. and Kohli, Pushmeet.
\newblock Memory bounded deep convolutional networks.
\newblock Technical report, University of Wisconsin-Madison, 2014.

\bibitem[Deng et~al.(2009)Deng, Dong, Socher, Li, Li, and
  Fei-Fei]{imagenet_cvpr09}
Deng, J., Dong, W., Socher, R., Li, L.-J., Li, K., and Fei-Fei, L.
\newblock {ImageNet: A Large-Scale Hierarchical Image Database}.
\newblock In \emph{CVPR}, 2009.

\bibitem[Denil et~al.(2013)Denil, Shakibi, Dinh, Ranzato, and
  de~Freitas]{DenilSDRF13}
Denil, Misha, Shakibi, Babak, Dinh, Laurent, Ranzato, Marc'Aurelio, and
  de~Freitas, Nando.
\newblock Predicting parameters in deep learning.
\newblock In \emph{NIPS}, pp.\  2148--2156, 2013.

\bibitem[Golub \& Van~Loan(1996)Golub and Van~Loan]{Golub:1996}
Golub, Gene~H. and Van~Loan, Charles~F.
\newblock \emph{Matrix Computations}.
\newblock Johns Hopkins University Press, 1996.

\bibitem[Gong et~al.(2014)Gong, Liu, Yang, and Bourdev]{gong2014compressing}
Gong, Yunchao, Liu, Liu, Yang, Ming, and Bourdev, Lubomir.
\newblock Compressing deep convolutional networks using vector quantization.
\newblock \emph{arXiv preprint arXiv:1412.6115}, 2014.

\bibitem[Graves et~al.(2015)Graves, Wayne, and Danihelka]{graves2015}
Graves, Alex, Wayne, Greg, and Danihelka, Ivo.
\newblock Neural {Turing} machines.
\newblock Technical report, Google DeepMind, 2015.

\bibitem[Han et~al.(2015{\natexlab{a}})Han, Mao, and
  Dally]{DBLP:journals/corr/HanMD15}
Han, Song, Mao, Huizi, and Dally, William~J.
\newblock Deep compression: {Compressing} deep neural network with pruning,
  trained quantization and {Huffman} coding.
\newblock \emph{arXiv preprint arXiv:1510.00149}, 2015{\natexlab{a}}.

\bibitem[Han et~al.(2015{\natexlab{b}})Han, Pool, Tran, and Dally]{Han2015}
Han, Song, Pool, Jeff, Tran, John, and Dally, William~J.
\newblock Learning both weights and connections for efficient neural networks.
\newblock In \emph{NIPS}, 2015{\natexlab{b}}.

\bibitem[Hermans \& Vaerenbergh(2015)Hermans and Vaerenbergh]{Hermans:2015}
Hermans, Michiel and Vaerenbergh, Thomas~Van.
\newblock Towards trainable media: Using waves for neural network-style
  training.
\newblock \emph{arXiv preprint arXiv:1510.03776}, 2015.

\bibitem[Hinton et~al.(2015)Hinton, Vinyals, and Dean]{Hinton15}
Hinton, Geoffrey~E., Vinyals, Oriol, and Dean, Jeffrey.
\newblock Distilling the knowledge in a neural network.
\newblock \emph{arXiv preprint arXiv:1503.02531}, 2015.

\bibitem[Hochreiter \& Schmidhuber(1997)Hochreiter and
  Schmidhuber]{hochreiter1997long}
Hochreiter, Sepp and Schmidhuber, J{\"u}rgen.
\newblock Long short-term memory.
\newblock \emph{Neural computation}, 9\penalty0 (8):\penalty0 1735--1780, 1997.

\bibitem[Huhtanen(2008)]{Huhtanen08}
Huhtanen, Marko.
\newblock Approximating ideal diffractive optical systems.
\newblock \emph{Journal of Mathematical Analysis and Applications},
  345:\penalty0 53--62, 2008.

\bibitem[Huhtanen \& Per{\"a}m{\"a}ki(2015)Huhtanen and
  Per{\"a}m{\"a}ki]{Huhtanen15}
Huhtanen, Marko and Per{\"a}m{\"a}ki, Allan.
\newblock Factoring matrices into the product of circulant and diagonal
  matrices.
\newblock \emph{Journal of Fourier Analysis and Applications}, 2015.

\bibitem[Jia et~al.(2014)Jia, Shelhamer, Donahue, Karayev, Long, Girshick,
  Guadarrama, and Darrell]{jia2014caffe}
Jia, Yangqing, Shelhamer, Evan, Donahue, Jeff, Karayev, Sergey, Long, Jonathan,
  Girshick, Ross, Guadarrama, Sergio, and Darrell, Trevor.
\newblock Caffe: Convolutional architecture for fast feature embedding.
\newblock \emph{arXiv preprint arXiv:1408.5093}, 2014.

\bibitem[Le et~al.(2013)Le, Sarl{\'o}s, and Smola]{le:2013}
Le, Quoc, Sarl{\'o}s, Tam{\'a}s, and Smola, Alex.
\newblock Fastfood -- approximating kernel expansions in loglinear time.
\newblock In \emph{ICML}, 2013.

\bibitem[Liu et~al.(2015)Liu, Wang, Foroosh, Tappen, and Pensky]{Liu_2015_CVPR}
Liu, Baoyuan, Wang, Min, Foroosh, Hassan, Tappen, Marshall, and Pensky,
  Marianna.
\newblock Sparse convolutional neural networks.
\newblock In \emph{CVPR}, 2015.

\bibitem[Makhoul(1980)]{makhoul1980fast}
Makhoul, John.
\newblock A fast cosine transform in one and two dimensions.
\newblock \emph{IEEE Transactions on Acoustics, Speech and Signal Processing},
  28\penalty0 (1):\penalty0 27--34, 1980.

\bibitem[M{\"u}ller-Quade et~al.(1998)M{\"u}ller-Quade, Aagedal, Beth, and
  Schmid]{muller1998algorithmic}
M{\"u}ller-Quade, J{\"o}rn, Aagedal, Harald, Beth, Th, and Schmid, Michael.
\newblock Algorithmic design of diffractive optical systems for information
  processing.
\newblock \emph{Physica D: Nonlinear Phenomena}, 120\penalty0 (1):\penalty0
  196--205, 1998.

\bibitem[Novikov et~al.(2015)Novikov, Podoprikhin, Osokin, and
  Vetrov]{novikov2015tensorizing}
Novikov, Alexander, Podoprikhin, Dmitry, Osokin, Anton, and Vetrov, Dmitry.
\newblock Tensorizing neural networks.
\newblock In \emph{NIPS}, 2015.

\bibitem[Reif \& Tyagi(1997)Reif and Tyagi]{reif97}
Reif, John and Tyagi, Akhilesh.
\newblock Efficient parallel algorithms for optical computing with the {DFT}
  primitive.
\newblock \emph{Applied Optics}, 36\penalty0 (29):\penalty0 7327--7340, 1997.

\bibitem[Romero et~al.(2015)Romero, Ballas, Kahou, Chassang, Gatta, and
  Bengio]{Romero2015}
Romero, Adriana, Ballas, Nicolas, Kahou, Samira~Ebrahimi, Chassang, Antoine,
  Gatta, Carlo, and Bengio, Yoshua.
\newblock {FitNets}: {Hints} for thin deep nets.
\newblock In \emph{ICLR}, 2015.

\bibitem[Saade et~al.(2015)Saade, Caltagirone, Carron, Daudet, Dremeau, Gigan,
  and Krzakala]{Saade:2015}
Saade, Alaa, Caltagirone, Francesco, Carron, Igor, Daudet, Laurent, Dremeau,
  Angelique, Gigan, Sylvain, and Krzakala, Florent.
\newblock Random projections through multiple optical scattering: Approximating
  kernels at the speed of light.
\newblock \emph{arXiv preprint arXiv:1510.06664}, 2015.

\bibitem[Sainath et~al.(2013)Sainath, Kingsbury, Sindhwani, Arisoy, and
  Ramabhadran]{SainathKSAR13}
Sainath, Tara~N., Kingsbury, Brian, Sindhwani, Vikas, Arisoy, Ebru, and
  Ramabhadran, Bhuvana.
\newblock Low-rank matrix factorization for deep neural network training with
  high-dimensional output targets.
\newblock In \emph{ICASSP}, pp.\  6655--6659, 2013.

\bibitem[Schmid et~al.(2000)Schmid, Steinwandt, Müller-Quade, Rötteler, and
  Beth]{Schmid:2000}
Schmid, Michael, Steinwandt, Rainer, Müller-Quade, Jörn, Rötteler, Martin,
  and Beth, Thomas.
\newblock Decomposing a matrix into circulant and diagonal factors.
\newblock \emph{Linear Algebra and its Applications}, 306\penalty0
  (1--3):\penalty0 131--143, 2000.

\bibitem[Sindhwani et~al.(2015)Sindhwani, Sainath, and
  Kumar]{sindhwani2015structured}
Sindhwani, Vikas, Sainath, Tara~N, and Kumar, Sanjiv.
\newblock Structured transforms for small-footprint deep learning.
\newblock In \emph{NIPS}, 2015.

\bibitem[Sukhbaatar et~al.(2015)Sukhbaatar, Szlam, Weston, and
  Fergus]{sukhbaatar2015}
Sukhbaatar, Sainbayar, Szlam, Arthur, Weston, Jason, and Fergus, Rob.
\newblock End-to-end memory networks.
\newblock In \emph{NIPS}, 2015.

\bibitem[Xu et~al.(2015)Xu, Ba, Kiros, Cho, Courville, Salakhutdinov, Zemel,
  and Bengio]{xu2015}
Xu, Kelvin, Ba, Jimmy, Kiros, Ryan, Cho, Kyunghyun, Courville, Aaron,
  Salakhutdinov, Ruslan, Zemel, Richard, and Bengio, Yoshua.
\newblock Show, attend and tell: Neural image caption generation with visual
  attention.
\newblock In \emph{ICML}, 2015.

\bibitem[Xue et~al.(2013)Xue, Li, and Gong]{XueLG13}
Xue, Jian, Li, Jinyu, and Gong, Yifan.
\newblock Restructuring of deep neural network acoustic models with singular
  value decomposition.
\newblock In \emph{{INTERSPEECH}}, pp.\  2365--2369, 2013.

\bibitem[Yang et~al.(2015)Yang, Moczulski, Denil, de~Freitas, Smola, Song, and
  Wang]{Yang2014}
Yang, Zichao, Moczulski, Marcin, Denil, Misha, de~Freitas, Nando, Smola, Alex,
  Song, Le, and Wang, Ziyu.
\newblock Deep fried convnets.
\newblock In \emph{ICCV}, 2015.

\end{thebibliography}
